\documentclass[11pt]{article}
\usepackage[margin=1.25in]{geometry}
\usepackage{authblk}
\usepackage[utf8]{inputenc}
\usepackage[T1]{fontenc}
\usepackage{stmaryrd}

\usepackage{bussproofs}
\EnableBpAbbreviations

\usepackage{graphbox}
\DeclareRobustCommand{\DLLogo}{%
  \begingroup\normalfont
  \kern-1.75pt\includegraphics[align=c,height=1.25\baselineskip]{dl}\kern-1.5pt%
  \endgroup
}

\usepackage{amsthm}
\newtheorem{theorem}{Theorem}

\newtheorem{lemma}{Lemma}
\newtheorem*{claim}{Claim}

\usepackage{tikz}
\usetikzlibrary{shapes,snakes}

\usepackage{amssymb}
\usepackage{amsmath}
\usepackage{xspace}
\usepackage[disable]{todonotes}
\usepackage{hyperref}
\usepackage{mathtools}
\usepackage{tikz}
\usepackage{graphicx}
\usepackage{csquotes}
\usepackage{pgfplots}
\usepackage{framed}
\usepackage{subcaption}
\usepackage{subcaption}

\usepackage{orcidlink}
\usepackage[htt]{hyphenat}

\newcommand{\ttodo}[4]{\ifthenelse{\equal{#1}{inline}}{\todo[inline,author=#2,color=#3]{#4}}{\todo[color=#3]{#2: #4}}}

\def\define#1#2#3%
{%
\renewcommand*{\do}[1]{%
 \expandafter\newcommand\csname
 #1\endcsname{#2}
}
\docsvlist{#3}
}

\define{#1}
{{\textsc{#1}}\xspace}
{Elk,Lethe,Fame,Evonne,Evee,HermiT,Spass,Capi,GraphStream}

\newcommand{\LetheAbduction}{\textsc{Lethe-Abduction}\xspace}

\newcommand{\wrt}{w.r.t.\ }
\newcommand{\ie}{i.e.\ }
\newcommand{\eg}{e.g.\ }

\define{#1mc}
{{\ensuremath{\mathcal{#1}}}\xspace}
{A,B,C,D,E,F,G,H,I,J,K,L,M,N,O,P,Q,R,S,T,U,V,W,X,Y,Z}

\define{#1}
{{\textsc{#1}}\xspace}
{NP,coNP,PSpace,ExpTime,ExpSpace}

\define{#1}
{{\ensuremath{\mathcal{#1}}}\xspace}
{ALC,EL,ALCOI,ALCH,ELH,ELI,SROIQ}

\newcommand{\cn}[1]{\textsf{#1}}

\newcommand{\ALCOImu}{\ensuremath{\ALCOI\mu}\xspace}

\newsavebox{\twosubbox}
\newcommand{\OverviewPic}{
\begin{figure}

\includegraphics[width=\textwidth]{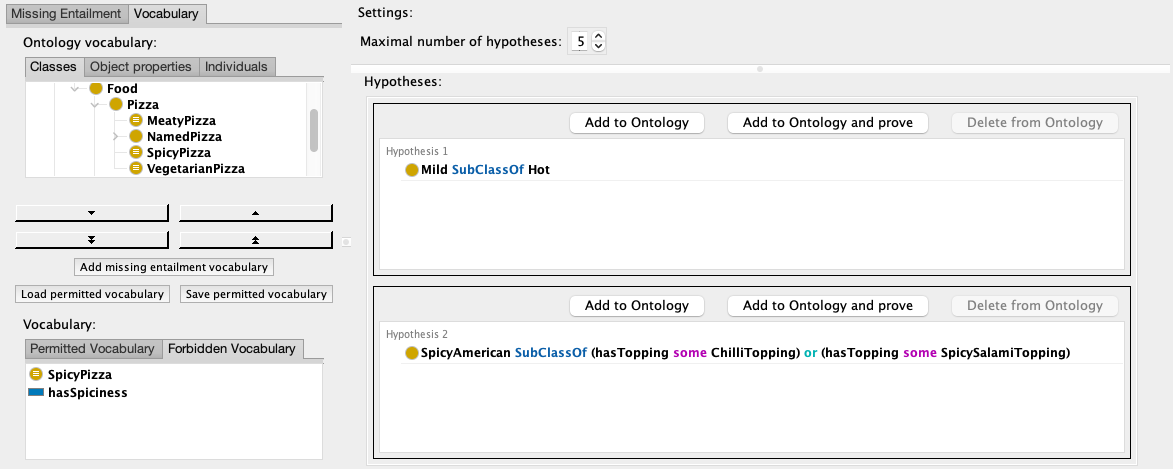}
\caption{Abduction results from \Lethe for $\emph{SpicyAmerican}\sqsubseteq\emph{SpicyPizza}$ with forbidden symbols \emph{SpicyPizza} and \emph{hasSpiciness}.}
\label{Figure:OverviewPic}

\end{figure}

}

\newcommand{\LetheResultPicture}{
	\begin{figure}
		\includegraphics[width=\textwidth]{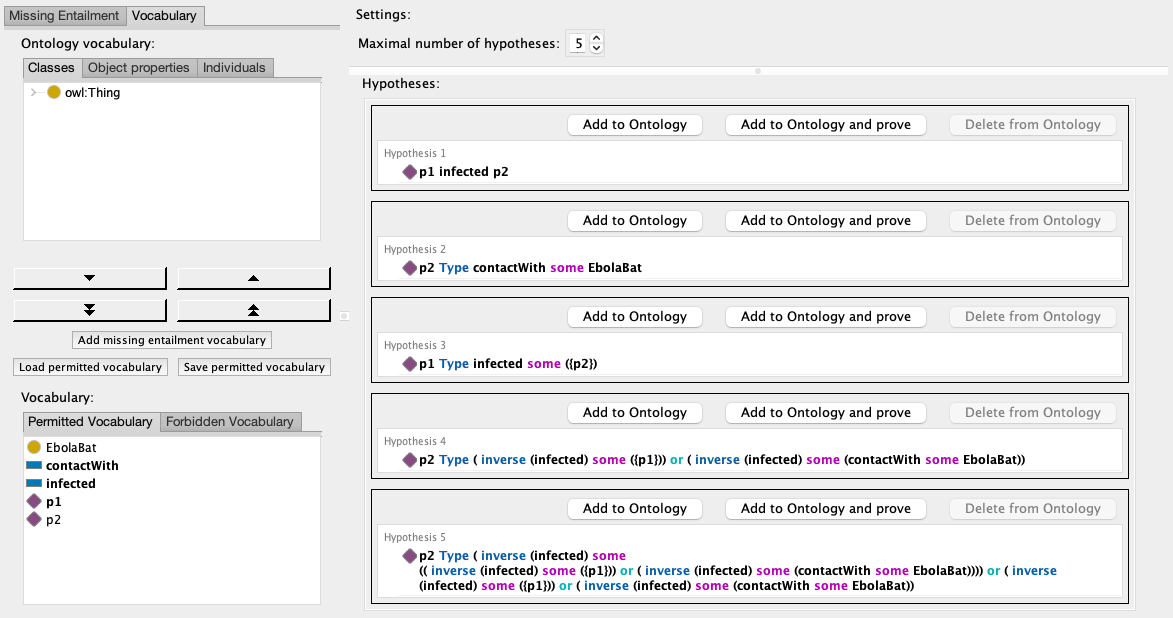}
		\centering
		\caption{An abduction result giving possible explanations for why patient
		\emph{p2} is an \emph{EbolaPatient}. This is based on the example used
		in~\cite{DBLP:conf/kr/KoopmannDTS20}.}
		\label{Figure:LetheResultPicture}
	\end{figure}
}

\newcommand{\nonent}{non-entailment\xspace}
\newcommand{\MisEnt}{Missing Entailment\xspace}
\newcommand{\misent}{missing entailment\xspace}
\newcommand{\Sig}{Vocabulary\xspace}
\newcommand{\sig}{vocabulary\xspace}

\newcommand{\obs}{missing entailment\xspace}
\newcommand{\explanationServiceInterface}{INonEntailmentExplanationService\xspace}
\newcommand{\explanationGenerationListener}{IExplanationGenerationListener\xspace}
\newcommand{\explainerInterface}{IOWLNonEntailmentExplainer\xspace}

\newcommand{\OBS}{\ensuremath{\Pmc}\xspace}

\newcommand{\evee}{Evee\xspace}
\newcommand{\plugin}{plug-in\xspace}

\newcommand{\plugins}{plug-ins\xspace}

\newcommand{\Protege}{Prot{\'e}g{\'e}\xspace}

\newcommand{\nonentAxiom}{\ensuremath{\eta}\xspace}

\newcommand{\arev}{\textsf{$\alpha$}\xspace}
\newcommand{\brev}{\textsf{$\beta$}\xspace}
\newcommand{\crev}{\textsf{$\Delta$}\xspace}
\newcommand{\drev}{\textsf{$\bar{\Delta}$}\xspace}

\newcommand{\highlighttt}[1]{
    \vspace{-7pt}
    \begin{center}
        \texttt{#1}
    \end{center}
    \vspace{-7pt}
}

\newcommand{\ce}{counterexample\xspace}
\newcommand{\Ce}{Counterexample\xspace}
\newcommand{\cegen}{counterexample generation\xspace}

\newcommand{\addDisj}{Add disjointnesses\xspace}
\newcommand{\recompute}{Recompute example\xspace}

\newcommand{\removeDisj}{Remove disjointnesses\xspace}
\newcommand{\addToOnt}{Add all to ontology\xspace}
\newcommand{\subclass}{SpicyAmerican\xspace}
\newcommand{\superclass}{SpicyPizza\xspace}

\begin{document}

\title{Why Not? Explaining Missing Entailments with \Evee (Technical Report)}

\author[1]{Christian Alrabbaa \orcidlink{0000-0002-2925-1765}}
\author[1]{Stefan Borgwardt \orcidlink{0000-0003-0924-8478}}
\author[1]{Tom Friese}
\author[2]{Patrick~Koopmann~\orcidlink{0000-0001-5999-2583}}
\author[1]{Mikhail Kotlov}
\affil[1]{Institute of Theoretical Computer Science, TU Dresden, Germany}
\affil[ ]{\nolinkurl{firstname.lastname@tu-dresden.de}}
\affil[2]{Department of Computer Science, VU Amsterdam, Netherlands}
\affil[ ]{\nolinkurl{p.k.koopmann@vu.nl}}

\date{}

\maketitle

\begin{abstract}
  Understanding logical entailments derived by a description logic reasoner is not always straight-forward
  for ontology users.
  For this reason, various methods for explaining entailments using justifications and proofs have been developed and implemented as \plugins for the ontology editor \Protege.
  However, when the user expects a missing consequence to hold, it is equally important to
  explain why it does not follow from the ontology.
  In this paper, we describe a new version of \Evee, a \Protege plugin that now also provides explanations for missing consequences, via existing and new techniques based on abduction and counterexamples.
\end{abstract}

\section{Introduction}

We present a Prot\'eg\'e plugin for explaining missing entailments from OWL ontologies.
The importance of explaining description logic reasoning to end-users has long been understood,
and has been studied in many forms over the past decades. Indeed, explainability is one of the main advantages of logic-based knowledge representations over sub-symbolic methods.
The first approaches to explain \emph{why} a consequence follows from a Description Logic (DL) ontology were based on step-by-step
\emph{proofs}~\cite{DBLP:conf/ecai/BorgidaFH00,DeMc-96}, but soon research focused on \emph{justifications}~\cite{DBLP:conf/ki/BaaderPS07,Horr-11,DBLP:conf/ijcai/SchlobachC03} that are easier to compute, but still very useful for pointing out the axioms responsible for an entailment.
Consequently, the ontology editor \Protege supports black-box methods for computing justifications for arbitrary OWL DL ontologies~\cite{DBLP:conf/sum/HorridgePS09}.
More recently, a series of papers investigated different methods of computing good proofs for entailments in DLs ranging from \EL to \ALCOI~\cite{DBLP:conf/dlog/KazakovKS17,DBLP:conf/lpar/AlrabbaaBBKK20,DBLP:conf/dlog/AlrabbaaBBKK20,DBLP:conf/cade/AlrabbaaBBKK21}, and the Protégé \plugins \texttt{proof-explanation}~\cite{DBLP:conf/dlog/KazakovKS17}
and \Evee~\cite{DBLP:conf/dlog/AlrabbaaBFK0P22},
as well as the web-based application~\Evonne~\cite{EVONNE},
were developed to make these algorithms available to ontology engineers.

While reasoning can sometimes reveal unexpected entailments that need explaining, very often the problem is not what is entailed, but what is \emph{not} entailed. In order to explain such missing entailments, and offer suggestions on how to repair them, both counterexamples and abduction have been suggested in the literature. A \emph{counterexample} is a model of the ontology that does not satisfy the entailment, which may be further augmented to focus the attention of the user to the part of the model that is most relevant for explaining the non-entailment~\cite{DBLP:conf/jist/AlrabbaaH22}. In \emph{abduction}, the non-entailment is explained by means of \emph{hypotheses}, which are sets of axioms that can be added to the ontology in order to entail the missing consequence~\cite{DBLP:conf/kr/KoopmannDTS20,DBLP:conf/ijcai/Koopmann21,Haifani2022CAPI}.
However, despite there being a lot of research on these explanation services of both theoretical and more practical form, so far, the tool 
support has not been integrated into standard ontology tools.

In this paper, we present version~0.2 of \Evee, a collection of plugins for the OWL ontology editor \Protege, which now also offers explanations for missing entailments.
Those plugins integrate the functionality provided by the external tools \Capi and \LetheAbduction for abduction, as well as the counterexample generation methods discussed in~\cite{DBLP:conf/jist/AlrabbaaH22}.
The explanations are provided by \Evee through a new \emph{\MisEnt Explanation} tab that contains a unified interface for explanations based on both counterexamples and abduction.
After specifying the missing entailment(s) and optionally a vocabulary for the explanation, the user can choose between different non-entailment explanation algorithms, which then provide either a graphical representation of a counterexample, or a list of different hypotheses to fix the missing entailments.
\Evee~0.2 has been tested with Java~8, OWL API 4.5.20, and \Protege 5.5.0, and can be downloaded and installed following the instructions at \url{https://github.com/de-tu-dresden-inf-lat/evee}.
The new plugins
depend on the external libraries \Capi,\footnote{\url{https://lat.inf.tu-dresden.de/~koopmann/CAPI}} \Spass,\footnote{\url{https://www.mpi-inf.mpg.de/departments/automation-of-logic/software/spass-workbench/classic-spass-theorem-prover}} and \LetheAbduction.\footnote{\url{https://lat.inf.tu-dresden.de/~koopmann/LETHE-Abduction}}

We describe the general interface of the new \MisEnt Explanation tab of \Evee in the next section, before explaining in detail the different 
explanation services and how they are accessed through the user interface. \Evee provides an infrastructure that makes it convenient for 
developers to develop new plugins based on their own methods for abduction or counterinterpretations. In 
Section~\ref{Section:DeveloperPerspective}, we explain how developers can use this infrastructure to provide new explanation services 
for missing explanations.

\section{Explanations for Non-Entailments}\label{Section:CorePlugin}

We assume the reader to be familiar with the syntax and semantics of DLs~\cite{book_introDL}.
The use case of our plugins is the following:
we have an active ontology~\Omc opened in the ontology editor \Protege, and there is a set of axioms~$\OBS$ that does not follow from~\Omc, \ie $\Omc\not\models\OBS$. The user may also specify a vocabulary $\Sigma$ to be used for the explanations, which is in particular useful for the abduction services.
The \emph{evee-protege-core} component provides the core functionality to specify~$\OBS$ and~$\Sigma$ and extension points for the actual explanation \plugins.
After installing the core plugin, a new tab is available via \emph{Window $\rightarrow$ Tabs $\rightarrow$ \MisEnt Explanation.}

\begin{figure}[t]
    \begin{center}
         \includegraphics[width=\textwidth]{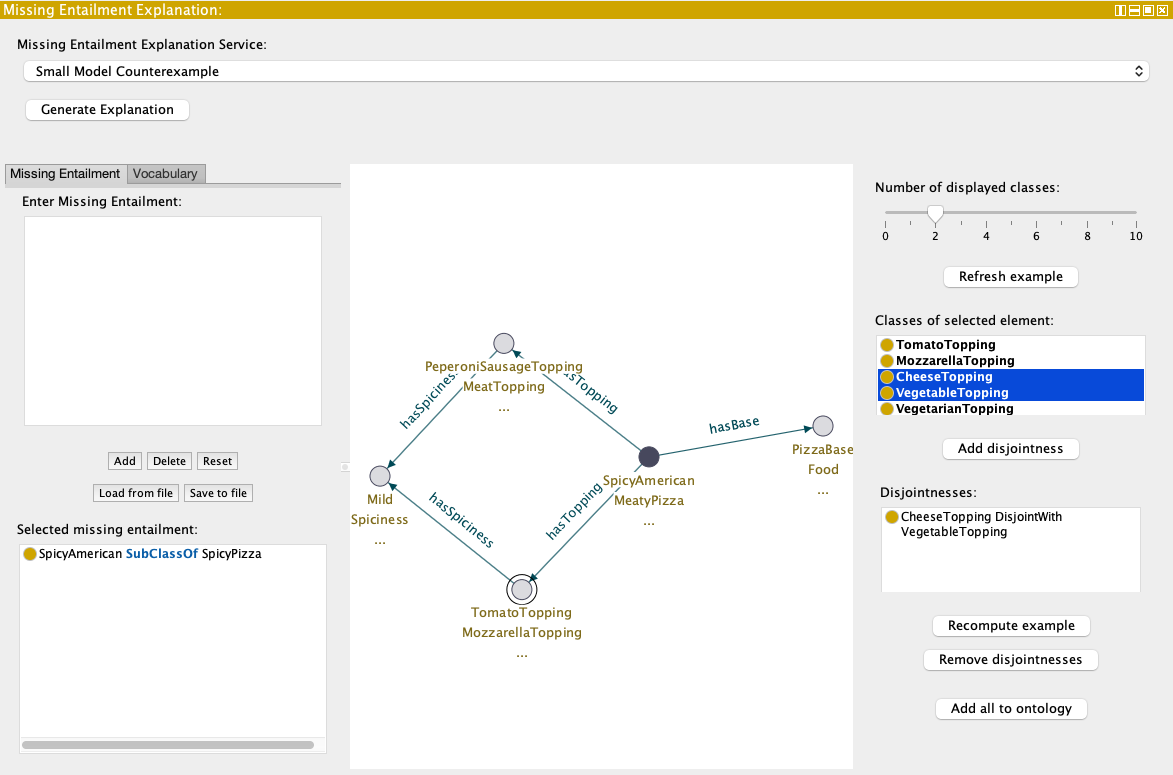}
        \caption{A \ce generated for \emph{\subclass} $\sqsubseteq $ \emph{\superclass} in an incomplete version of the pizza ontology. The example illustrates two problems with the ontology: The circled element at the bottom shows that \emph{MozarellaTopping} and \emph{TomatoTopping}, as well as \emph{CheeseTopping} and \emph{VegetableTopping} are not disjoint from one another.}
        \label{fig:ce}
    \end{center}
\end{figure}

Figure~\ref{fig:ce} shows this tab in action.
It is divided into three major parts:
In the upper part, one of the installed \emph{\misent explanation services} can be chosen, the computation process can be started, and general information is displayed.
On the left, the \obs and \sig can be entered, and in the center, the explanation will be displayed.
The explanation view depends on the selected explanation service (see Sections~\ref{Section:Counterexamples} and~\ref{Section:Abduction}).
If the entered \obs and \sig are not supported by the service,
the \emph{Generate Explanation} button at the top will be disabled and an explanatory message will be shown.
How missing entailments are entered can be see on the left in Figure~\ref{fig:ce}.
The text field at the top can be used to enter individual axioms.
The buttons in the middle allow the user to \emph{add} an axiom to the list below, \emph{remove} a selected axiom, or \emph{reset} the whole list.
Only \emph{OWL logical axioms} are allowed, \eg subclass-, equivalence-, and disjointness axioms and assertions.
The selected \misent can also be saved to or loaded from an OWL ontology file, which can be useful for demonstration purposes.

By selecting the \emph{\Sig} tab, the users can restrict the vocabulary used by the explanations, which can be seen in Figure~\ref{Figure:OverviewPic}.
Here, the \emph{Ontology \Sig} of the currently active ontology can be accessed via the class hierarchy, object property hierarchy, and list of individuals.
The \sig of the explanation will be restricted to the names in the tab \emph{Permitted \Sig} on the bottom, while \emph{Forbidden \Sig} shows the remaining names of the ontology \sig.
Depending on the currently opened tabs, the arrows and the button \emph{Add missing entailment \sig} can be used to add or remove names to and from the selected \sig.
Again, the permitted \sig can be saved to and loaded from an external file.
By default, the whole ontology \sig is \emph{permitted}, but this can be changed in the \plugin preferences at \emph{Preferences $\rightarrow$ Explanations $\rightarrow$ \MisEnt $\rightarrow$ General}.

While the explanation is generated, a progress window is used to indicate the computation status and show additional information.
The computation can be canceled by closing the window or clicking the \emph{Cancel} button.
This will show a separate cancelation window while the computation is being terminated.

As a running example to illustrate the different explanation services, we consider a modified, incomplete version of the
Pizza ontology.\footnote{\url{http://protege.stanford.edu/ontologies/pizza/pizza.owl}} This version is missing some axioms to make it entail $\emph{\subclass}\sqsubseteq\emph{\superclass}$. It will turn out that other things are missing in this ontology as well, and the plugin will help the user in adding those missing parts.

\section{Counterexamples}\label{Section:Counterexamples}

The first obvious way to explain the missing entailment is to show an example of a \subclass that is not a \superclass, as shown in Figure~\ref{fig:ce}.
\evee 2.0 includes two \plugins that provide \cegen services: the \emph{Small Model \Ce~Generator} and
 the \emph{Relevant \Ce~Generator} using \Elk~\cite{DBLP:journals/jar/KazakovKS14}. %
For a missing entailment $C\sqsubseteq D$, a counterexample is a model of the ontology that contains an element
that belongs to~$C$, but does not belong to~$D$.
The \plugins visualize counterexamples as directed graphs, where the nodes are individuals labeled by concept names
and the edges are labeled by role names.
The Small Model \Ce Generator is developed for $\mathcal{EL}_\bot$, which supports disjointness axioms.
This generator generates complete models, but tries to reduce the number of elements to keep the model small.
The ELK Relevant \Ce~Generator instead focuses on relevant fragments of models, using the methods described in~\cite{DBLP:conf/jist/AlrabbaaH22}. It was developed for the description logic $\mathcal{EL}$.
Both generators require a single GCI to be entered as non-entailment, and support arbitrary vocabularies~$\Sigma$.
Since we only explain non-entailed GCIs, the counterexample generators ignore the ABox of the active ontology, and only consider the TBox.
The generated counterexample only shows names from~$\Sigma$. We first describe the counterexample view, before explaining the different methods in detail.
To visualize the counterexamples, we use GraphStream,\footnote{\url{https://graphstream-project.org}}
a Java library for modeling, analyzing and visualizing graphs.
Its functionality allows not only to visualize models, but also to dynamically make changes to models when the ontologies change.
In the generated graphs, domain elements are depicted as circles.
We highlight elements that are of particular importance for understanding the generated model. For instance, each counterexample contains a \emph{root element}, marked in black, which satisfies the concept on the left-hand side of the GCI to be explained.
 For readability, only some of the concept names for each domain element are shown, whose number can be adapted using the \emph{Number of displayed classes} slider on the right panel of the counterinterpretation view.
When the user selects a node, the \emph{Classes of selected element} list in the right panel displays all the concept names to which the selected element belongs. In the node selected in Figure~\ref{fig:ce}, the user notices in this way an element that is both a \emph{TomatoTopping} and a \emph{MozzarellaTopping}, pointing at another bug in the ontology---but more on this later.

The graphical model view allows zoom to facilitate exploring large graphs, and users can move the nodes of the graph.
Dragging the mouse over the background canvas navigates through the graph.
To make the graphical representation of a \ce more informative, we display concept names in the order from more specific to less specific.
In Figure~\ref{fig:ce}, we display that the root element belongs only to \emph{\subclass} and \emph{MeatyPizza}, but it implicitly is also a \emph{Pizza}, a \emph{Food}, and ultimately a \emph{DomainThing}, since the first two classes are subsumed by them.
However, if we instead displayed in the graph that this element is a \emph{DomainThing}, we would give no useful information about the element.

As in the present example, the visualized model can also reveal missing disjointness axioms in the ontology.
As already noticed,
in Figure~\ref{fig:ce}, the selected element is both a
\emph{MozzarellaTopping} and a \emph{TomatoTopping}.
The reason is a missing disjointness between \emph{CheeseTopping} and \emph{VegetableTopping}.
The right panel allows the user to add new disjointness axioms as needed and visualize the result.
For this, the user selects the corresponding concept names in the \emph{Classes of selected element} list and presses \emph{\addDisj}. A 
disjointness axiom with the selected names is then added to the \emph{Disjointnesses} list, as shown in the figure.
By pressing the \emph{\recompute} button, the user gets shown an updated model with the new disjointness applied.
If the user is not satisfied with the changes to the model resulting from the new axioms, they can be deleted using the \emph{\removeDisj} button.
Finally, axioms from the \emph{Disjointnesses} list can be added to the active ontology with a click of the \emph{\addToOnt} button.

\subsection{Small Model \Ce~Generator}\label{Section:Counterexamples:smallModel}

In this explanation service, \ce{}s are generated using a tableau-based algorithm for the description logic $\mathcal{EL}_\bot$.
Given a GCI $C\sqsubseteq D$,
 the algorithm first initializes an ABox $\mathcal{A}$ containing as only axiom $C(a^*)$, where $a^*$ is a fresh individual name.
Next, it adds $D \sqsubseteq B^*$ to the TBox, where $B^*$ is a fresh concept name, and normalizes~\cite{book_introDL} the TBox.
Note that $\mathcal{T}\models C \sqsubseteq D$ iff $\mathcal{T}\cup \{D\sqsubseteq B^*\}\models C \sqsubseteq B^*$.
Thus, if $a^{*\mathcal{I}}\not\in B^{*\mathcal{I}}$ in the generated model $\mathcal{I}$, then~$a^{*\mathcal{I}}\not\in D^\mathcal{I}$.
Therefore, the generated model is a \ce iff $a^{*\mathcal{I}}\not\in B^{*\mathcal{I}}$~\cite{book_introDL}.

The model of the ontology is obtained using a complete and clash-free ABox $\mathcal{A'}$ obtained from the ABox $\mathcal{A}$ by an exhaustive application of the expansion rules from Table \ref{t:rules}.
The $\sqsubseteq$-rule is almost identical to the similar rule in algorithms for $\mathcal{ALC}$.
The only difference is that it takes into account the structure of the normalized TBox.
It becomes applicable only to concept assertions with concept names or with a concept name under an existential restriction.
The $\sqcap$- and $\exists_1$-rules are designed to add assertions that make the $\sqsubseteq$-rule applicable.
For individual names having a successor belonging to some concept name, the $\exists_1$-rule creates a concept assertion with this concept name under an existential restriction.
The $\sqcap$-rule breaks conjunctions into simpler assertions.

To keep the model small, we reuse existing individuals as successors when trying to satisfy existential role restrictions.
Before reusing an individual name, a consistency check is performed, so that the rule cannot introduce any inconsistency.
Moreover, we only reuse an individual as successor if this does not make the root element an instance of $B^*$, since the aim is to construct a counterexample for $C\sqsubseteq B^*$.

The expansion rules are applied exhaustively, but the $\exists_2$-rule is applied only if no other rule is applicable.
This restriction reduces the number of applications of the $\exists_2$-rule, and consequently the number of individuals added.
The algorithm %
iteratively applies the rules
until no more rule is applicable, and then translates the resulting ABox into an interpretation in the usual way.
The correctness of the algorithm is shown in the appendix.

\begin{table}%
    \centering
    \caption{Expansion rules of the tableau method generating small counterexamples for $\EL_\bot$. Here, At($D$) refers to the conjuncts of the concept $D$, or to the singleton set $\{D\}$ if $D$ is not a conjunction.
    }\label{t:rules}
\begin{tabular}{p{0.14\textwidth}|p{0.74\textwidth}}

 \textbf{$\sqcap$-rule} & \textbf{if} $\mathcal{A}$ contains $D(a)$,  but not $C(a)$,  $C \in$ At($D$) \textbf{then}  $\mathcal{A} \longrightarrow \mathcal{A} \cup\{a:C\}$ \\
\hline
 \textbf{$\exists_{1}$-rule} & \textbf{if} $\mathcal{A}$ contains $r(a,b)$ and $A(b)$, $A$ is a concept name or $\top$,  but not $\exists r.A(a)$  $\mathcal{A} \longrightarrow \mathcal{A} \cup\{\exists r.A(a)\}$ \\
\hline
 \textbf{$\exists_{2}$-rule} & \textbf{if} $\mathcal{A}$ contains $\exists r.E(a)$, but there is no $b$ s.t. $r(a,b)$ and $E(b)$

~~\textbf{if} there is some $c$, s.t. $ \mathcal{T}\cap \mathcal{A} \cup\{r(a,c),E(c)\}$ is consistent and does not entail $B^*(a^*)$
\textbf{then} $\mathcal{A} \longrightarrow\mathcal{A} \cup\{r(a,c),E(c)\}$

 ~~\textbf{else} $\mathcal{A} \longrightarrow\mathcal{A} \cup\{r(a,d),E(d), \top(d)\}$, where $d$ is new in $\mathcal{A}$
 \\
\hline
 \textbf{The $\sqsubseteq$-rule} & \textbf{if} $A(a)\in \mathcal{A}$, $A \sqsubseteq B\in \mathcal{T}$ or $\{A_{1}(a),A_{2}(a)\}\subseteq \mathcal{A}$, $A_{1} \sqcap A_{2} \sqsubseteq B\in \mathcal{T}$ or $\exists r.A(a)\in \mathcal{A}$, $\exists r.A \sqsubseteq B\in \mathcal{T}$ and $B(a)\not\in\Amc$.

\textbf{then} $\mathcal{A} \longrightarrow \mathcal{A} \cup\{B(a)\}$   \\

\end{tabular}

\end{table}

\subsection{Relevant \Ce~Generator}
A more focussed explanation to missing entailments is provided by the 
\emph{Relevant \Ce~Generator}. \emph{Relevant counterexamples} explain
missing entailments by showing relevant parts of the models of \EL ontologies, 
where this time \emph{canonical models}~\cite{book_introDL} are used.
Canonical models have two properties that are beneficial for
explanations. First, they reuse domain elements, \ie when a concept $C$
appears multiple times in a TBox~\Tmc, the substructure of the canonical model
$\Imc_\Tmc$ satisfying~$C$ is reused. This makes
$\Imc_\Tmc$ a compact interpretation. Second, for any \nonent
$\nonentAxiom$, $\Tmc \not\models \nonentAxiom$ iff $\Imc_\Tmc
\not\models \nonentAxiom$, and hence ~$\Imc_\Tmc$ directly serves as a counterexample.
However, the size of these models can still be large. To overcome this, we focus
on certain parts of the model, since in general not the entire model is relevant for
the explanation of the current~\nonentAxiom.

We distinguish four types of relevance as shown
in~\cite{DBLP:conf/jist/AlrabbaaH22},
which define the \emph{\arev-, \brev-, \crev-}, and \emph{\drev-relevant 
parts} of $\Imc_\Tmc$.
One possible
explanation for $\Tmc\not\models C\sqsubseteq D$ is to show the user an element that satisfies~$C$, does not satisfy $D$, and satisfies
all axioms in \Tmc. This element serves as a \emph{witness} for the
\nonent, and together with its required successors
forms the \emph{\arev-relevant part} of the canonical model.
Another possibility is to contrast $C$ with $D$, by including also a representative element satisfying~$D$, which gives rise to the
\emph{\brev-relevant part} of the canonical~model.

The \crev-relevant part is a refinement of the \brev-relevant part that focuses on the conditions that are imposed by the ontology on~$D$, but not on~$C$.
This allows for an \enquote{explanation by contradiction} as follows. If $\Tmc \models C \sqsubseteq D$, then
every subsumer~$E$ of $D$ must also subsume~$C$. However, there is
a model of \Tmc (the canonical model) in which there is an element of $C$ that 
intuitively does not satisfy some condition~$E$ that is satisfied by every element 
of $D$. Hence,~$C$ cannot be subsumed by $D$ \wrt \Tmc. Therefore, the
\crev-relevant part contains only those elements illustrating the contrasting
conditions~$E$, \eg $r$-successors (not) satisfying~$F$ in case that
$E=\exists r.F$.
The \drev-relevant part is a further refinement of the \crev-relevant part that tries to generalize these conditions $E$. For example, if $\Tmc \models D
\sqsubseteq E$, $\Tmc \not \models C \sqsubseteq E$ and $E = \exists
r.\exists r.\exists r.F$, then it is sufficient to consider $\exists
r.\top$ instead of $E$, assuming  that $\Tmc \not \models C \sqsubseteq
\exists r.\top$.
For more details we refer the reader to~\cite{DBLP:conf/jist/AlrabbaaH22}.

\newcommand{\lethe}{LETHE\xspace}
\newcommand{\spass}{SPASS\xspace}
\newcommand{\capi}{CAPI\xspace}

\OverviewPic

\section{Abduction}\label{Section:Abduction}

Counterexamples always focus only on one model, and they do not necessarily make it obvious what needs to be done to fix a missing entailment. This is where the explanation services based on abduction come into play.
For our running example, we show an explanation based on abduction in Figure~\ref{Figure:OverviewPic}.
Evee~0.2 includes two \plugins %
based on \emph{abduction}, namely the \emph{Complete Signature-Based Abduction solver} based on \Lethe~\cite{DBLP:journals/ki/Koopmann20}%
\footnote{
 \url{https://lat.inf.tu-dresden.de/~koopmann/LETHE-Abduction}}
and the \emph{Connection-Minimal Abduction solver} utilizing \capi~\cite{Haifani2022CAPI}.\footnote{
\url{https://lat.inf.tu-dresden.de/~koopmann/CAPI}}
Given a \nonent{} $\Omc\not\models\OBS$,
abduction computes a set of \emph{hypotheses}~$\Hmc$, which are sets of axioms such that $\Omc\cup\Hmc\models\OBS$.
Without further restrictions, $\OBS$ is already a hypothesis, which
is why usually additional constraints on the solution space are given. \emph{Signature-based abduction}~\cite{DBLP:conf/ijcai/Koopmann21} relies on a user-given vocabulary. The signature-based abduction service computes a set of alternative hypotheses using only names from the vocabulary, such that any other such hypothesis can be obtained by strengthening
or combining those hypotheses.  In contrast, \capi computes hypotheses satisfying a minimality criterion called \emph{connection-minimality}~\cite{Haifani2022CAPI}, with the aim of focussing on those hypotheses that have a more direct connection to the observation. The signature-based explanations support the DL \ALC, observations can be a mix of several ABox and TBox axioms, and the hypotheses can make arbitrary use of DL constructs, which in particular means that the result can be an unbounded sequence of hypotheses. Connection minimal explanations support \EL ontologies, entailments consisting of a single GCI, and hypotheses are always without role restrictions.
To use \capi, the FOL theorem prover \spass %
needs to be installed separately.
We require an adapted version of \spass, which can be installed following the instructions on the web page of \capi.%
\footnote{
When using the \capi abduction \plugin for the first time, it will ask for the directory that \spass was installed to.
This directory can later be changed in the \Protege preferences, see Section~\ref{Section:Abduction:CapiAbductionSolver}.}

After computing an explanation with an abduction solver, one or more hypotheses will be displayed in a list, as shown in Figures \ref{Figure:OverviewPic} and~\ref{Figure:LetheResultPicture}.
Depending on the input and the algorithm, the number of results may differ and may even be infinite.
Therefore, the user can specify the number of new results that are added to the list whenever the \emph{Generate Explanation} button is clicked again.
Using additional buttons shown at each hypothesis, the user can then easily add the hypothesis to the ontology (to repair the \nonent) and get an explanation why the hypothesis entails the \nonent, using the proof functionality provided by the \texttt{proof-explanation} and \Evee~0.1 \plugins~\cite{DBLP:conf/dlog/KazakovKS17,DBLP:conf/dlog/AlrabbaaBFK0P22}.
The third button can be used to revert the changes to the ontology.
Each service resets the displayed results if any changes are made to the active ontology, unless these changes are made via these \emph{Add} and \emph{Delete} buttons.

\subsection{Complete Signature-Based Abduction}\label{Section:Abduction:LetheAbudctionSolver}
Signature based hypotheses are computed by the abduction extension of the external library \lethe~\cite{DBLP:conf/kr/KoopmannDTS20,DBLP:journals/ki/Koopmann20}.
We extended the original method by an additional, equivalence-preserving simplification step to make the hypotheses more user-friendly.
\LetheResultPicture
The method computes so-called \emph{complete signature-based hypotheses}, which are hypotheses that are fully in the signature, and which generalize any other possible such hypothesis. This is only possible by using
disjunctions and least fixpoint operators, which is why the output of this method is
a disjunction of the form
\(
\bigvee\limits_{i=1}^{n}\big(\bigwedge\limits_{j=1}^{m} \alpha_{i,j}\big)
\),
with each $\alpha_{i,j}$ an $\ALCOImu$ axiom. Intuitively, each disjunct is an alternative hypothesis, but their axioms may include
\emph{least fixpoint concepts} of the form $\mu X.C[X]$~\cite{DBLP:conf/ijcai/CalvaneseGL99}.
To obtain from this disjunction a sequence of hypotheses that can be displayed in \Protege, the fixpoint concepts
need to be \emph{unraveled}.
This is done in order of increasing role depth, \ie the shallowest hypotheses are shown first.
For example, in Figure~\ref{Figure:LetheResultPicture},
hypotheses 4 and 5 are obtained by unravelling of the following assertion, followed by some syntactic reformulations:
\begin{align*}
                     p_2: \mu X.\exists\cn{infected}^-.\big(\exists\cn{contactWith}.\cn{EbolaBat}
                     \sqcup \{p_1\}\sqcup X\big)
\end{align*}

\subsection{Capi Abduction solver}\label{Section:Abduction:CapiAbductionSolver}
The \capi\ abduction solver internally relies on the FOL theorem prover \spass to compute the solutions to an abduction problem.
In particular, based on a translation into first-order logic clauses, \spass computes a set of prime implicates, which
are then used by the Java component of the tool to construct the different
hypotheses (see~\cite{Haifani2022CAPI} for details).
 The \Protege plugin takes some additional input parameters that can be configured by the user.
By default, \Spass stops generating prime implicates after a time limit of
10 seconds is reached. This is usually sufficient to obtain a large set of hypotheses, but if the results are unsatisfactory, the time limit
can be changed under \emph{Preferences $\rightarrow$ Explanations $\rightarrow$ \MisEnt $\rightarrow$ Connection-Minimal Abduction (CAPI)}.
Further options concern the post-processing of solutions generated by \spass, which were not included in the original
implementation presented in~\cite{Haifani2022CAPI}, but later added for convenience: 1) explanations can be simplified by removing redundant axioms, 2) axioms can be simplified by
removing redundant conjuncts or disjuncts, and 3) hypotheses can be ordered by specificity, which means: if one hypothesis implies another one, the implied hypothesis is shown later.
Without these post-processing steps, hypotheses may be long and generally contain long lists of conjunctions, which is why the optimizations are turned on by default.
On the other hand, by deactivating all post-processing steps, we obtain
hypotheses that are faithful to the method described in~\cite{Haifani2022CAPI}.

\section{Adding New Non-Entailment Explanation Services}\label{Section:DeveloperPerspective}
For developers who want to add their own \nonent explanation services, the module \emph{evee-protege-core} %
provides two new extension points for \Protege \plugins:
\highlighttt{de.tu\_dresden.inf.lat.evee.nonEntailment\_explanation\_service}
for explanation services %
and
\highlighttt{de.tu\_dresden.inf.lat.evee.nonEntailment\_preferences}
for managing \plugin-specific preference settings.
The preferences extension point is simple:
One implements the interface \texttt{PreferencesPanel} provided by \Protege, and the resulting panel will be displayed in a tabbed pane accessible via \emph{Preferences $\rightarrow$ Explanations $\rightarrow$ \MisEnt}.
Using the explanation service extension point requires a few more steps.
Essentially, an explanation service needs to provide the \nonent explanations as a Java Stream and visualize the elements of this stream in \Protege.
To facilitate this for abduction and counterexamples, we provide two abstract base classes for abduction and counterexample services.
The main interface for explanation services is %
\texttt{\explainerInterface},
whose most important methods are \texttt{supportsExplanation()} and \texttt{generateExplanations()}.
The first method
determines %
whether the \emph{Generate Explanation} button should be enabled or disabled.
Since this method is called whenever the input changes, its implementation should not be computationally expensive.
The second method returns the %
explanation in the form of a \texttt{Stream<Set<OWLAxiom>{}>},
where each set represents a single explanation for the \misent.\footnote{
For abduction, these sets are the hypotheses, and for counterexamples they are sets of assertions that describe models.}
We use streams to accommodate a potentially infinite number of explanations, as in the case of signature-based abduction.

On top of these generic methods, \texttt{\explanationServiceInterface} provides functionality to connect to the user interface.
The method \texttt{computeExplanation()} is called when the \emph{Generate Explanation} button is clicked, and \texttt{cancel()} is called when
the user wants to cancel.
The explanation service can also use an \texttt{IProgressTracker} to send information to the loading window and an \texttt{\explanationGenerationListener} to send events to the main tab.
To the loading window, one can send the current progress as well as a String describing the current computation status.
The events for the listener can have an \texttt{ExplanationEventType} of \texttt{COMPUTATION\_COMPLETE}, \texttt{RESULT\_RESET}, \texttt{WARNING}, or \texttt{ERROR}.
This allows the explanation service to display new results, clear the shown result, or display warnings or errors, respectively.
The main tab ultimately requires the result in the form of a \texttt{java.awt.Component}, which is retrieved via the method \texttt{getResult} right after an event of type \texttt{COMPUTATION\_COMPLETE} is received.
This way, each service enjoys a great degree of freedom in displaying its explanation. We already provide pre-built functionality for abduction and counterexample services, as described in the following sections.

\subsection{Abstract Counterexample Generation Service}\label{Section:Counterexample:AbstractCounterexampleGenerationService}

The class \texttt{AbstractCounterexampleGenerationService} contains all functionality related to the visualization of counterexamples and implements all methods of the \texttt{INon\-Entailment\-Explanation\-Service} interface.
Classes extending \texttt{Abstract\-Counter\-example\-Gene\-ra\-tion\-Service} differ primarily in the  used \ce generator, which must be specified in the constructor using the method \texttt{setCounterexample\-Generator()}.

Each \ce generator implements the \texttt{IOWLCounterexampleGenerator} interface.
The interface extends \texttt{IOWLNon\-EntailmentExplainer} by \texttt{generateModel()} and \texttt{getMarkedIndividuals()}.
The model returned by \texttt{generateModel()} is represented using a set of \texttt{OWLIndividualAxiom}s. Each of those should be an instance of either \texttt{OWL\-Class\-Asser\-tion\-Axiom} or \texttt{OWLObjectPro\-pertyAssertionAxiom},
which specify the content of the classes and properties in the interpretation.
These axioms are also returned by the method \texttt{generateExplanations()} of the interface \texttt{IOWLCounterexampleGenerator}.
Finally, using the method \texttt{getMarkedIndividuals()}, the service can specify individual names that will be highlighted in the visualization of the model.
As an example of how the abstract counterexample generator operates, consider again the algorithm described in Section~\ref{Section:Counterexamples:smallModel}.
This algorithm is implemented in a separate counterexample generator and executed when \texttt{generateModel()} is called.
Afterwards, the abstract counterexample generator sends an \texttt{ExplanationEvent} of type \texttt{COMPUTATION\_COMPLETE} to the main tab.
The resulting counter example is then provided to the main tab via the method \texttt{generateExplanations()}.

\subsection{Abstract Abduction Solver}\label{Section:Abduction:AbstractAbductionSolver}
The class \texttt{AbstractAbductionSolver} %
is used by both of the \plugins presented in Section~\ref{Section:Abduction}.
The main responsibilities of this class are caching the results computed for a specific input, creating and maintaining the actual result component that is displayed to the user, and handling user input when any of the \emph{Add}- or \emph{Delete}-buttons of a hypothesis are clicked (see Figure~\ref{Figure:LetheResultPicture}).
The class is generic in order to facilitate the caching of different kinds of results for each implementing solver via its generic type parameter.
Caching is not done automatically by the abstract solver class.
Instead, the implementing solver can use the methods \texttt{checkResultInCache}, \texttt{saveResultToCache} and \texttt{loadResultFromCache}. %
In contrast to these user-experience-related functionalities, the actual computation of the \misent explanation is left to the individual implementations of the abstract class.
As explained above, implementing the interface \texttt{\explainerInterface} requires providing a stream of explanations via the method \texttt{generateExplanations()}, \ie a stream of sets of OWLAxioms, where each set represents a single hypothesis.
This method will ultimately be called by the \texttt{AbstractAbductionSolver} when creating the list of \nonent explanations that is shown to the user.

\section{Conclusion}

We believe that our \plugins are an important step towards making reasoning
more understandable to ontology users. The implementation is still relatively new
and there are little performance issues that need to be solved. We hope that
our framework will encourage other developers to implement their own
explanation services in \Evee. In addition to further improving \Evee, we
would like to evaluate our \plugins in a user study. It would also be
interesting to investigate whether \Evee can be used to improve
university-level teaching on ontologies and description logics.

 \paragraph{Acknowledgments}
This work was supported by the DFG grant 389792660 as part of TRR~248
(\url{https://perspicuous-computing.science}).

\pagebreak

\pagebreak
\appendix
\section{Proofs for Section~\ref{Section:Counterexamples:smallModel}}

\begin{figure}[t]

\noindent\fbox{%
    \parbox{\textwidth}{%
$\Delta ^{\mathcal{I}} = \{a\mid C(a) \in \mathcal{A} '\}$

$a^{\mathcal{I}} = a$ for each individual name $a$ occurring in $\mathcal{A'}$

$A^{\mathcal{I}} =\{a\mid A \in a:A\in\Amc'\}$ for each concept name occurring in $\mathcal{A'}$

$r^{\mathcal{I}}  = \left\{( a,b) \mid ( a,b) :r\in A^{\mathcal{I}}\right\}$ for each role name $r$ occurring in $\mathcal{A'}$
    }%
}
\caption{The model $\mathcal{I}$ induced from a complete ABox $\mathcal{A'}$.}
\label{fig:model}

\end{figure}
Figure \ref{fig:model} defines the model of $\mathcal{O'}=\mathcal{A'}\cup\mathcal{T}$ returned by the Algorithm \textsf{Generate-model}($\mathcal{T}$) based on the complete ABox $\mathcal{A'}$.

\begin{lemma}
 \label{lem:pres}
For each consistent $\mathcal{EL}_{\bot}$ ontology $\mathcal{O}=\mathcal{A}\cup\mathcal{T}$ with its TBox $\mathcal{T}$ being normalized and for each expansion rule, the ontology $\mathcal{O'}=\mathcal{A'}\cup\mathcal{T}$ obtained after the rule application is consistent.

\end{lemma}

\begin{proof}
Let $\mathcal{I}$ be a model of $\mathcal{O}$ before the rule application,
for each expansion rule we show that $\mathcal{I}$ is a model of $\mathcal{O'}=\mathcal{A'}\cup\mathcal{T}$ obtained after the rule application.

\begin{description}

\item [\textbf{The $\sqcap$-rule}.]
If $A_1(a)\sqcap A_{2}\in \mathcal{A}$, then $a^{\mathcal{I}}\in (A_{1} \sqcap A_{2})^{\mathcal{I}}$, then $a^{\mathcal{I}}\in A_{1}^{\mathcal{I}} \cap A_{2}^{\mathcal{I}}$, and then $a^{\mathcal{I}}\in A_{1}^{\mathcal{I}}$ and  $a^{\mathcal{I}}\in A_{2}^{\mathcal{I}}$. So $\mathcal{I}$ is a model of $\mathcal{A}\cup\{A_1(a),A_2(a)\}$.
\item [\textbf{The $\exists_{1}$-rule}.]
If $r(a,b)$ and $B(b)$ in $\mathcal{A}$, then $(a^{\mathcal{I}},b^{\mathcal{I}})\in r^{\mathcal{I}}$ and $b^{\mathcal{I}}\in B^{\mathcal{I}}$. By induction, $a^{\mathcal{I}}\in  (\exists r.B)^{\mathcal{I}} $.
The $\exists_{1}$-rule adds  $\exists r.B(a)$ to $\mathcal{A}$. $\mathcal{I}$ is still a model of $\mathcal{A} \cup \{\exists r.B(a)\}$.
\item [\textbf{The $\sqsubseteq$-rule}.]
If $a:A\in \mathcal{A}$ and $A\sqsubseteq B \in \mathcal{T}$, $a^{\mathcal{I}}\in A^{\mathcal{I}}$ and $a^{\mathcal{I}}\in B^{\mathcal{I}}$. So $\mathcal{I}$ is a model of $\mathcal{A}\cup\{a:B\}$.

\item [\textbf{The $\exists_{2}$-rule}.]
Let $C$ be an arbitrary $\mathcal{EL}$ concept.
If $\exists r.C(a) \in \mathcal{A}$, then $a^{\mathcal{I}}\in  (\exists r.C)^{\mathcal{I}} $.
Thus, there is $x$, s.t. $(a^{\mathcal{I}},x)\in r^{\mathcal{I}}$ and $x\in C^{\mathcal{I}}$.
There are two possible cases of an application of the $\exists_{2}$-rule:
\begin{enumerate}
    \item There is some $b$ s.t.\ $\mathcal{O} \cup \{r(a,b),C(b)\}$ is not inconsistent.
    So, the $\exists_{2}$-rule adds $r(a,b)$.
    By the definition of the rule, the ontology is still consistent.
    \item There is no such $b$ s.t.\ $\mathcal{O} \cup \{r(a,b),C(b)\}$ is consistent.
    Then $c$ is created and $\{r(a,c)$, $C(c)$, $c:\top\}$ is added to $\mathcal{A}$.
    If we say that $c^{\mathcal{I}} = x$, then $\mathcal{I}$ is a model of $\mathcal{A} \cup \{r(a,c)$, $C(c)$, $ c:\top\}$.

\end{enumerate}
\end{description}

For each expansion rule, the interpretation  $\mathcal{I}$ is still a model of $\mathcal{O'}=\mathcal{A'}\cup\mathcal{T}$ obtained after the rule application.
Therefore, the ontology $\mathcal{O'}$ obtained after the rule application is consistent.
\end{proof}

\begin{lemma}[Termination]
\label{lem:term} For each consistent normalized $\mathcal{EL}_{\bot}$ TBox $\mathcal{T}$, the algorithm \textsf{Generate-model}($\mathcal{T}$) terminates.
\end{lemma}
\begin{proof}Let $\Omc$ be an ontology containing the TBox $\Tmc$ and an ABox $\Amc$ initialized as described in Section \ref{Section:Counterexamples:smallModel}.
Also, let $m$ be $|\textsf{sub}(\mathcal{O})|$ and $n$  be the number of concepts in $\mathcal{O}$.
Termination follows from the following properties:

\begin{enumerate}
\item For a given individual $a$, we can have only a finite number of rule applications. The reasons for that are:

\begin{enumerate}
    \item The expansion rules never delete an assertion.
    \item The $\sqcap$-rule, the $\exists_{2}$-rule,the $\sqsubseteq$-rule can only add a new assertion of the form $C(a)$ for $C \in$ sub($\mathcal{O}$).
    \item The $\exists_{1}$-rule can only add a new assertion of the form $\exists r.A(a)$ for $A $ being a concept name.
\end{enumerate}
So, for a given individual we can have at most  $m$ + $n$ rule applications that add a concept assertion.
\item The number of individuals in the resulting ABox $\mathcal{A'}$ is finite.
\begin{enumerate}
\item Because the size of $\mathcal{A}$ is finite, it can contain only a finite number of individuals.
\item For a given individual, the number of successors, generated by applications of the $\exists_{2}$-rule is finite.
\begin{claim}
The $\exists_{1}$-rule cannot add an assertion which can trigger the $\exists_{2}$-rule.
\end{claim}
\textit{Proof of the Claim.} the $\exists_{1}$-rule is triggered by $\{r(a,b),A(b)\}\subseteq\mathcal{A}$, s.t. $A$ is a concept name or $\top$, and adds $\exists r.A(a)$. the $\exists_{2}$-rule is triggered only if $\exists r.A(a)\in\mathcal{A}$ but $\{r(a,b),A(b)\}$ not in $\mathcal{A}$.
Therefore, it can never be applicable because the action of the $\exists_{2}$-rule in this case is equal to the condition of the $\exists_{1}$-rule.\hfill$\blacksquare$

Therefore, for a given individual, the number of successors generated by applications of the $\exists_{2}$-rule is bounded by $m$ because each individual can belong to only $m$ concepts, that can trigger the $\exists_{2}$-rule.
\item For a given individual, the depth of the chain of successors generated by applications of the $\exists_{2}$-rule is bounded.
For any individual $a$, any path along its successors can contain at most $2^{m}$ individuals before it contains individual names $b$ and $ c$ such that \textsf{con}$\index{\mathcal{A}}$($b$) = \textsf{con}$\index{\mathcal{A}}$($c$).
If $c$ was created but $b$ was not reused,
the ontology $\mathcal{O'}=\mathcal{A'}\cup\mathcal{T} $, where $\mathcal{A'}$ is the ABox where $b$ is used as the successor, is inconsistent.
But because \textsf{con}$_{\mathcal{A}}$($b$) = \textsf{con}$_{\mathcal{A}}$($c$), for the ABox $\mathcal{A}$ where $c$ was created,  $\mathcal{O}=\mathcal{A}\cup\mathcal{T} $ is also inconsistent.
No expansion rule can bring inconsistency, as shown in Lemma \ref{lem:pres}.
Therefore, the input ontology should be inconsistent, which contradicts the initial assumption.

We obtain that for a giving individual, the depth of the chain of successors generated by applications of the $\exists_{2}$-rule is bounded by $2^{m}$.

\end{enumerate}
\end{enumerate}
The algorithm can generate only a finite number of individuals and for each of them the number of rule applications is bounded.
Therefore, the algorithm terminates in a finite number of rule applications.
\end{proof}

\begin{lemma}
\label{lem:sound}
Let $\Omc$ be an ontology containing the TBox $\Tmc$ and an ABox $\Amc$ initialized as described in Section \ref{Section:Counterexamples:smallModel}.
Assume, $\mathcal{T}$ is normalized, then the interpretation $\mathcal{I}$ returned by \textsf{Generate-model}($\mathcal{T}$) is a model of $\mathcal{O}$.
\end{lemma}

\begin{proof}
To prove Lemma \ref{lem:sound}, we first show that the interpretation $\mathcal{I}$ returned by the algorithm is a model of $\mathcal{O'}=\mathcal{A'}\cup\mathcal{T}$, where $\mathcal{A'}$ is the complete ABox, obtained by the algorithm.
\begin{description}
\item \emph{$\mathcal{I}$ is a model of every assertion in $\mathcal{A'}$}.
\begin{description}
    \item [role assertions:] $r(a,b) \in \mathcal{A'}$.  $(a^{\mathcal{I}},b^{\mathcal{I}})= (a,b) \in r^{\mathcal{I}}$ by the definition of $\mathcal{I}$.
    \item [concept assertions:]$C(a)\in \mathcal{A'}$, where $C$ is an arbitrary concept.
    We show that $a^{\mathcal{I}}\in C^{\mathcal{I}}$ by induction on the structure of $C$:
    \begin{description}
    \item [$C = A$.]
    If $a:A\in \mathcal{A'}$,  $a^{\mathcal{I}}\in A^{\mathcal{I}}$ by the definition of $\mathcal{I}$.
    \item [$C = A \sqcap B$.]
    Completeness of $\mathcal{A'}$ yields that  $\{a:A,B(b)\} \subseteq\mathcal{A} $, otherwise the $\sqcap$-rule would be applicable.
    By the definition of $\mathcal{I}$, $a^{\mathcal{I}}\in A^{\mathcal{I}}$ and $a^{\mathcal{I}}\in B^{\mathcal{I}}$, induction yields that $a^{\mathcal{I}} \in (A \cap B)^{\mathcal{I}} $.
    \item [$C = \exists r. D$.]
    Completeness of $\mathcal{A'}$ yields that there is some $b$ s.t. $\{b:D,r(a,b)\} \subseteq\mathcal{A'} $.
    By the definition of $\mathcal{I}$, $b\in D^{\mathcal{I}}$ and $ (a^{\mathcal{I}},b^{\mathcal{I}}) \in r^{\mathcal{I}}$, induction yields $a^{\mathcal{I}} \in \exists r. D^{\mathcal{I}} $.
    \end{description}

\end{description}

\item \emph{$\mathcal{I}$ is a model of every GCI in $\mathcal{T}$}.
If $\mathcal{I}$ also is a model of $\mathcal{T}$ , GCI's in $\mathcal{T}$ should be satisfied.
Let $C$ and $D$ be arbitrary $\mathcal{EL}$ concepts that can appear in a normalized TBox.
We show soundness by showing that whenever a domain element $a^{\mathcal{I}}$ belongs to $C^{\mathcal{I}}$ and $C\sqsubseteq D \in \mathcal{T}, a^{\mathcal{I}}$ also belongs to $ D^{\mathcal{I}} $.
    \begin{description}
    \item [$C = A$.]
    If $a^{\mathcal{I}}\in A^{\mathcal{I}}$, then $a:A\in \mathcal{A'}$. By completeness of $\mathcal{A'}$, $D(a)$ is also in $\mathcal{A'}$, otherwise the $\sqsubseteq$-rule would be applicable.
    Thus, by the definition of the model $\mathcal{I}$, $a^{\mathcal{I}}\in D^{\mathcal{I}}$.
    \item [$C = A_{1} \sqcap A_{2}$.]
    If $a^{\mathcal{I}}\in (A_1 \cap A_2)^{\mathcal{I}}$, then $a^{\mathcal{I}}\in (A_1 )^{\mathcal{I}}$ and $a^{\mathcal{I}}\in (A_2 )^{\mathcal{I}}$, which  yields that  $\{A_1(a),A_2(a)\} \subseteq\mathcal{A'} $.
     By completeness of $\mathcal{A'}$, $D(a)$ is also in $\mathcal{A'}$ otherwise the $\sqsubseteq$-rule would be applicable.
    Then, by the definition of the model $\mathcal{I}$, $a^{\mathcal{I}}\in D^{\mathcal{I}}$

    \item [$C = \exists r. B$.]
    If $a^{\mathcal{I}}\in (\exists r. B )^{\mathcal{I}}$, then there is some $b$, s.t. $(a^{\mathcal{I}},b^{\mathcal{I}}) \in r^{\mathcal{I}}$
    and $b^{\mathcal{I}} \in B^{\mathcal{I}}$.
  Therefore, $\{r(a,b),B(b)\} \subseteq\mathcal{A'} $. By completeness of $\mathcal{A'}$, $\exists r.B(a)$ is also in $\mathcal{A'}$ otherwise the $\exists_1$-rule would be applicable and $D(a)$ is also in $\mathcal{A'}$ otherwise the $\sqsubseteq$-rule would be applicable.
    Then, by the definition of the model $\mathcal{I}$, $a^{\mathcal{I}}\in D^{\mathcal{I}}$
    \end{description}

\end{description}

Because the expansion rules do not delete assertions, $\mathcal{A} \subseteq\mathcal{A'} $ and $\mathcal{I}$ is a model of $\mathcal{O} = \mathcal{A}\cup\mathcal{T}$.
\end{proof}

We show soundness of the algorithm by contrapositive:
\begin{lemma}[Soundness]
\label{lem:sound2}

If the input normalized TBox $\mathcal{T}\not\models C\sqsubseteq D$, then in the interpretation $\mathcal{I}$ generated by \textsf{Generate-model}($\mathcal{T}$), $a^{*\mathcal{I}}\not\in B^{*\mathcal{I}}$.
\end{lemma}
\begin{proof} Lemmas \ref{lem:term}, \ref{lem:sound} show that the algorithm generates an interpretation $\mathcal{I}$ in a finite number of steps, and that this interpretation is indeed a model of $\mathcal{O}= \mathcal{T}\cup \mathcal{A}$, where $\mathcal{A}=\{a^*:C\}$.
We also know that if $\mathcal{A}\cup \mathcal{T}\not\models C\sqsubseteq D$ then $\mathcal{A}\cup \mathcal{T}\not\models B^*(a^*)$. To prove that in the interpretation $\mathcal{I}$, the root individual $a^{*\mathcal{I}}$ does not belong to $B^{*\mathcal{I}}$, we show that no rule can add $B^*(a^*)$ if $\mathcal{A}\cup \mathcal{T}\not\models B^*(a^*)$.

\begin{description}
\item [\textbf{The $\sqcap$-rule}.]
If $A_{1}\sqcap A_{2}(a)\in \mathcal{A}$, the application of the $\sqcap$-rule adds $\{A_1(a),A_2(a)\}$. Assume that $a=a^*$ and $A_{1}=B^*$ then $B^*(a^*)$ is added. But then $\mathcal{A}\cup\mathcal{T}\models B^*(a^*)$ because $B^*\sqcap A2\sqsubseteq_\emptyset B^*$. Therefore, this rule can add $B^*(a^*)$ iff $\mathcal{A}\cup\mathcal{T}\models B^*(a^*)$.
\item [\textbf{The $\exists_{1}$-rule}.]
This rule can not add a concept name.
\item [\textbf{The $\sqsubseteq$-rule.}]
If $a:A\in \mathcal{A}$, $A \sqsubseteq B\in \mathcal{T}$ or $\{A_1(a),A_2(a)\}\subseteq \mathcal{A}$, $A_{1} \sqcap A_{2} \sqsubseteq B\in \mathcal{T}$ or $\exists r.A(a)\in \mathcal{A}$, $\exists r.A \sqsubseteq B\in \mathcal{T}$, the $\sqsubseteq$-rule adds $a:B$. Assume that $a=a^*$ and $B=B^*$, then $B^*(a^*)$ is added. But then, if $a^*:A\in \mathcal{A}$, $A \sqsubseteq B^*\in \mathcal{T}$, then $\mathcal{A}\cup\mathcal{T}\models B^*(a^*)$. The same is true if $\exists r.A(a)\in \mathcal{A}$, $\exists r.A \sqsubseteq B^*\in \mathcal{T}$. If $\{a^*:A_{1},a^*:A_{2}\}\subseteq \mathcal{A}$, $A_{1} \sqcap A_{2} \sqsubseteq B^*\in \mathcal{T}$, then $\mathcal{A}\cup\mathcal{T}\models A_1\sqcap A_2(a^*)$ because in any model $\mathcal{I}$ of $\mathcal{A}$, $a^{*\mathcal{I}} \in A^\mathcal{I}_1 $ and $a^{*\mathcal{I}} \in A^\mathcal{I}_2 $, therefore $a^{*\mathcal{I}} \in (A_1\cap A_2)^\mathcal{I} $,
and finally $\mathcal{A}\cup\mathcal{T}\models B^*(a^*)$ because $A_1 \sqcap A_2 \sqsubseteq_\mathcal{T}B^*$. The $\sqsubseteq$-rule can add $B^*(a^*)$ iff $\mathcal{A}\cup\mathcal{T}\models B^*(a^*)$.

\item [\textbf{The $\exists_{2}$-rule.}]
This rule can not add $B^*(a^*)$ if $\mathcal{A}\cup \mathcal{T}\not\models B^*(a^*)$ by the definition of the $\exists_{2}$-rule.
\end{description}

That shows that no rule application can add $B^*(a^*)$, unless $\mathcal{A}\cup\mathcal{T}\models B^*(a^*)$ and, which is equivalent, $\mathcal{A}\cup\mathcal{T}\models C\sqsubseteq
 D$. According to the definition of the model $\mathcal{I}$, $a^{*\mathcal{I}}\in B^{*\mathcal{I}}$ only if $B^*(a^*)\in \mathcal{A'}$, which is not the case because the input ABox $\mathcal{A}$ did not contain $B^*(a^*)$ according to its definition and no rule could add $B^*(a^*)$ to it. Therefore, $a^{*\mathcal{I}}\not\in B^{*\mathcal{I}}$.
\end{proof}

We show completeness also by contrapositive:
\begin{lemma}[Completeness]
\label{lem:comp2}

If in the generated by \textsf{Generate-model}($\mathcal{T}$) interpretation $\mathcal{I}$, the element $a^{*\mathcal{I}}$ is in $ B^{*\mathcal{I}}$, then the input normalized TBox
$\mathcal{T}\not\models C\sqsubseteq D$.
\end{lemma}
\begin{proof}
    Let $\Omc$ be an ontology containing the TBox $\Tmc$ and an ABox $\Amc$ initialized as described in Section \ref{Section:Counterexamples:smallModel}.
The proof of Lemma~\ref{fig:model} establishes that $\mathcal{I}$ is a model of $\mathcal{O}$.  Then we have a model of $\mathcal{O}$ in which $a^{\mathcal{I}}$ is in $ C^{\mathcal{I}}$ but not in $ B^{\mathcal{I}}$. Therefore, $C\not\sqsubseteq_\mathcal{T}B^*$, and because $C\sqsubseteq_\mathcal{T}B^*$ iff $C\sqsubseteq_\mathcal{T}D$, $ C\sqsubseteq_\mathcal{T}D$ does not hold.
\end{proof}

\begin{theorem}

For any $\mathcal{EL}_\bot$ concepts $C$ and $D$, normalized $\mathcal{EL}_\bot$ TBox $\mathcal{T}$, $C\sqsubseteq_\mathcal{T} D$ iff ${a^*}^\mathcal{I}\in {B^*}^\mathcal{I}$ in the model $\mathcal{I}$, returned by the algorithm \textsf{Generate-model}($\mathcal{T}$).
\end{theorem}
\begin{proof}
Both if and only if direction as well as termination hold, as shown in Lemmas~\ref{lem:sound2},~\ref{lem:comp2} and~\ref{lem:term}.
\end{proof}


\begin{thebibliography}{10}
	
	\bibitem{DBLP:conf/lpar/AlrabbaaBBKK20}
	C.~Alrabbaa, F.~Baader, S.~Borgwardt, P.~Koopmann, and A.~Kovtunova.
	\newblock Finding small proofs for description logic entailments: {T}heory and
	practice.
	\newblock In E.~Albert and L.~Kov{\'{a}}cs, editors, {\em {LPAR} 2020: 23rd
		International Conference on Logic for Programming, Artificial Intelligence
		and Reasoning}, volume~73 of {\em EPiC Series in Computing}, pages 32--67.
	EasyChair, 2020.
	
	\bibitem{DBLP:conf/dlog/AlrabbaaBBKK20}
	C.~Alrabbaa, F.~Baader, S.~Borgwardt, P.~Koopmann, and A.~Kovtunova.
	\newblock On the complexity of finding good proofs for description logic
	entailments.
	\newblock In S.~Borgwardt and T.~Meyer, editors, {\em Proceedings of the 33rd
		International Workshop on Description Logics {(DL} 2020)}, volume 2663 of
	{\em {CEUR} Workshop Proceedings}. CEUR-WS.org, 2020.
	
	\bibitem{DBLP:conf/cade/AlrabbaaBBKK21}
	C.~Alrabbaa, F.~Baader, S.~Borgwardt, P.~Koopmann, and A.~Kovtunova.
	\newblock Finding good proofs for description logic entailments using recursive
	quality measures.
	\newblock In A.~Platzer and G.~Sutcliffe, editors, {\em Automated Deduction -
		{CADE} 28 - 28th International Conference on Automated Deduction}, volume
	12699 of {\em Lecture Notes in Computer Science}, pages 291--308. Springer,
	2021.
	
	\bibitem{DBLP:conf/dlog/AlrabbaaBFK0P22}
	C.~Alrabbaa, S.~Borgwardt, T.~Friese, P.~Koopmann, J.~M{\'{e}}ndez, and
	A.~Popovic.
	\newblock On the eve of true explainability for {OWL} ontologies: {D}escription
	logic proofs with {E}vee and {E}vonne.
	\newblock In O.~Arieli, M.~Homola, J.~C. Jung, and M.~Mugnier, editors, {\em
		Proceedings of the 35th International Workshop on Description Logics (DL)},
	volume 3263 of {\em {CEUR} Workshop Proceedings}. CEUR-WS.org, 2022.
	
	\bibitem{DBLP:conf/jist/AlrabbaaH22}
	C.~Alrabbaa and W.~Hieke.
	\newblock Explaining non-entailment by model transformation for the description
	logic {$\mathcal{EL}$}.
	\newblock In A.~Artale, D.~Calvanese, H.~Wang, and X.~Zhang, editors, {\em
		Proceedings of the 11th International Joint Conference on Knowledge Graphs,
		{IJCKG}}, pages 1--9. {ACM}, 2022.
	
	\bibitem{book_introDL}
	F.~Baader, I.~Horrocks, C.~Lutz, and U.~Sattler.
	\newblock {\em An Introduction to Description Logic}.
	\newblock Cambridge University Press, 2017.
	
	\bibitem{DBLP:conf/ki/BaaderPS07}
	F.~Baader, R.~Pe{\~{n}}aloza, and B.~Suntisrivaraporn.
	\newblock Pinpointing in the description logic \emph{EL}\({}^{\mbox{+}}\).
	\newblock In J.~Hertzberg, M.~Beetz, and R.~Englert, editors, {\em {KI} 2007:
		Advances in Artificial Intelligence, 30th Annual German Conference on AI},
	volume 4667 of {\em Lecture Notes in Computer Science}, pages 52--67.
	Springer, 2007.
	
	\bibitem{DBLP:conf/ecai/BorgidaFH00}
	A.~Borgida, E.~Franconi, and I.~Horrocks.
	\newblock Explaining {ALC} subsumption.
	\newblock In W.~Horn, editor, {\em {ECAI} 2000, Proceedings of the 14th
		European Conference on Artificial Intelligence}, pages 209--213. {IOS} Press,
	2000.
	
	\bibitem{DBLP:conf/ijcai/CalvaneseGL99}
	D.~Calvanese, G.~D. Giacomo, and M.~Lenzerini.
	\newblock Reasoning in expressive description logics with fixpoints based on
	automata on infinite trees.
	\newblock In T.~Dean, editor, {\em Proceedings of the Sixteenth International
		Joint Conference on Artificial Intelligence, {IJCAI}}, pages 84--89. Morgan
	Kaufmann, 1999.
	
	\bibitem{Haifani2022CAPI}
	F.~Haifani, P.~Koopmann, S.~Tourret, and C.~Weidenbach.
	\newblock Connection-minimal abduction in {$\mathcal{EL}$} via translation to
	{FOL}.
	\newblock In J.~Blanchette, L.~Kov{\'a}cs, and D.~Pattinson, editors, {\em
		Automated Reasoning}, pages 188--207, Cham, 2022. Springer International
	Publishing.
	
	\bibitem{Horr-11}
	M.~Horridge.
	\newblock {\em Justification Based Explanation in Ontologies}.
	\newblock PhD thesis, University of Manchester, UK, 2011.
	
	\bibitem{DBLP:conf/sum/HorridgePS09}
	M.~Horridge, B.~Parsia, and U.~Sattler.
	\newblock Explaining inconsistencies in {OWL} ontologies.
	\newblock In L.~Godo and A.~Pugliese, editors, {\em Scalable Uncertainty
		Management, Third International Conference, {SUM}, Proceedings}, volume 5785
	of {\em Lecture Notes in Computer Science}, pages 124--137. Springer, 2009.
	
	\bibitem{DBLP:conf/dlog/KazakovKS17}
	Y.~Kazakov, P.~Klinov, and A.~Stupnikov.
	\newblock Towards reusable explanation services in protege.
	\newblock In A.~Artale, B.~Glimm, and R.~Kontchakov, editors, {\em Proceedings
		of the 30th International Workshop on Description Logics}, volume 1879 of
	{\em {CEUR} Workshop Proceedings}. CEUR-WS.org, 2017.
	
	\bibitem{DBLP:journals/jar/KazakovKS14}
	Y.~Kazakov, M.~Kr{\"{o}}tzsch, and F.~Simancik.
	\newblock The incredible {ELK} - from polynomial procedures to efficient
	reasoning with {$\mathcal{EL}$} ontologies.
	\newblock {\em J.\ Autom.\ Reason.}, 53(1):1--61, 2014.
	
	\bibitem{DBLP:journals/ki/Koopmann20}
	P.~Koopmann.
	\newblock {LETHE:} forgetting and uniform interpolation for expressive
	description logics.
	\newblock {\em K{\"{u}}nstliche Intell.}, 34(3):381--387, 2020.
	
	\bibitem{DBLP:conf/ijcai/Koopmann21}
	P.~Koopmann.
	\newblock Signature-based abduction with fresh individuals and complex concepts
	for description logics.
	\newblock In Z.~Zhou, editor, {\em Proceedings of the Thirtieth International
		Joint Conference on Artificial Intelligence, {IJCAI}}, pages 1929--1935.
	ijcai.org, 2021.
	
	\bibitem{DBLP:conf/kr/KoopmannDTS20}
	P.~Koopmann, W.~Del{-}Pinto, S.~Tourret, and R.~A. Schmidt.
	\newblock Signature-based abduction for expressive description logics.
	\newblock In D.~Calvanese, E.~Erdem, and M.~Thielscher, editors, {\em
		Proceedings of the 17th International Conference on Principles of Knowledge
		Representation and Reasoning, {KR}}, pages 592--602, 2020.
	
	\bibitem{DeMc-96}
	D.~L. McGuinness.
	\newblock {\em Explaining Reasoning in Description Logics}.
	\newblock PhD thesis, Rutgers University, NJ, USA, 1996.
	
	\bibitem{EVONNE}
	J.~Méndez, C.~Alrabbaa, P.~Koopmann, R.~Langner, F.~Baader, and R.~Dachselt.
	\newblock Evonne: A visual tool for explaining reasoning with {OWL} ontologies
	and supporting interactive debugging.
	\newblock {\em Computer Graphics Forum}, 2023.
	
	\bibitem{DBLP:conf/ijcai/SchlobachC03}
	S.~Schlobach and R.~Cornet.
	\newblock Non-standard reasoning services for the debugging of description
	logic terminologies.
	\newblock In G.~Gottlob and T.~Walsh, editors, {\em IJCAI-03, Proceedings of
		the Eighteenth International Joint Conference on Artificial Intelligence},
	pages 355--362. Morgan Kaufmann, 2003.
	
\end{thebibliography}
\end{document}